\documentclass[12pt,a4paper]{article}      
 \usepackage{microtype}
\usepackage{graphicx}
\usepackage{enumerate}
\usepackage{amssymb, amsmath, amsthm}
\usepackage{fullpage}

\usepackage{tikz}
\usetikzlibrary{arrows,automata}

\newtheorem{proposition}{Proposition}[section] 

\newtheorem{theorem}[proposition]{Theorem}
\newtheorem{lemma}[proposition]{Lemma}
\newtheorem{corollary}[proposition]{Corollary}
\newtheorem{remark}[proposition]{Remark}

\usepackage{dsfont}

\usepackage[plainpages=false,pdfpagelabels]{hyperref}
\usepackage[T1]{fontenc}
\usepackage{lmodern}

\newcommand{\N}{\ensuremath{{\mathbb N}}}

\newcommand{\R}{\ensuremath{{\mathbb R}}}

\DeclareMathOperator*{\argmax}{argmax}

\begin{document}

\title{Classification error in multiclass discrimination from Markov data}


\author{
Sören Christensen\thanks{Department of Mathematical Sciences, Chalmers University of Technology and G\"oteborg
University, SE-412 96 Göteborg, Sweden; E-mail: sorenc@chalmers.se}
, Albrecht Irle\thanks{Mathematical Institute, University of Kiel, D-24098 Kiel, Germany
E-mail: \textit{lastname}@math.uni-kiel.de}
, and Lars Willert\footnotemark[2]} 
 \maketitle
\begin{abstract}
As a  model for an on-line classification setting we consider
 a stochastic process  $(X_{-n},Y_{-n})_{n}$,
 the present time-point being denoted by 0, with observables
$ \ldots,X_{-n},X_{-n+1},\ldots, X_{-1}, X_0$
from which the pattern $Y_0$ is to be inferred.
So in this classification setting, in addition to the present
observation $X_0$ a number $l$ of preceding observations may be used
for classification, thus taking a possible dependence structure into account
as it occurs e.g. in an ongoing classification of handwritten characters. We treat the question how the performance of classifiers is improved by using such additional information. For our  analysis, a
 hidden Markov model is used. Letting $R_l$ denote the minimal risk of misclassification using $l$ preceding
observations we show that the difference $\sup_k |R_l - R_{l+k}|$ decreases
exponentially fast as $l$ increases. This suggests that a small $l$ might
already lead to a noticeable improvement. To follow this point we look at the use of
past observations for kernel classification rules.

Our practical findings  in simulated hidden Markov models and in the classification of handwritten characters indicate that using $l=1$, i.e.
just  the last preceding observation in addition to $X_0$, can  lead to a
substantial reduction of the risk of misclassification. So, in the presence of stochastic dependencies, we
advocate to use
$ X_{-1},X_0$ for finding the pattern $Y_0$ instead of only $X_0$ as one would in the independent situation.

\end{abstract}

\section{Introduction}

In pattern recognition, the following basic situation is considered:
A random variable $(X,Y)$ consists of an observed pattern  $X\in  {\cal X} $, typically $ {\cal X} = \R^d$,
from which we want to infer the unobservable class
$Y$ which belongs to a given finite set $M$ of classes.
Consider the case that the
 distribution $P^{(X,Y)}$ is known. Then the classification rule  which
chooses the class having maximum a posteriori probability
given the observed pattern $x$ has minimal risk of misclassification. This optimal rule
is given by
$$ x \mapsto \argmax_{y} P(Y=y|X=x)$$
where $\argmax$ takes, in a measurable way, some value $y^*$ with $P(Y=y^*|X=x)= \max_{y}P(Y=y|X=x)$.
The minimal risk of misclassification, often termed  the Bayes risk, is given by
$$R= \int \min_{y} P(Y \neq y|X=x) P^X(dx).$$
Even though in many problems of pattern recognition the distribution of $P^{(X,Y)}$ will not be known,
the Bayes risk is a quantity of major importance as it provides the benchmark behaviour against which
any other procedure is judged.

Let us briefly recall the i.i.d. model of supervised learning which has provided a main direction of research, see, e.g., the monograph \cite{gyorfi1996probabilistic}. There, in addition to $(X,Y)$, we have a learning sequence $(X_1',Y_1'),(X_2',Y_2'),...,(X_n',Y_n')$ of independent copies of $(X,Y)$, i.e. having the same distribution. This sequence is sampled independently of $(X,Y)$ and is used for learning proposes, in a statistical sense for the estimation of unknown distributions to construct the classification procedure. 

In this paper we take a different approach which is motivated by an on-line classification setting which we model in the following way:
 There is given
 a stochastic process  $(X_{-n},Y_{-n})_{n}$,
 the present time-point being denoted by 0, with observables in temporal order
$$ \ldots,X_{-n},X_{-n+1},\ldots, X_{-1}, X_0$$
from which the pattern $Y_0$ is to be inferred. The time parameter $n$ belongs to some  set of the form
$\{0,1,\ldots,m\}$ or, for  mathematical purposes, to $\N_0= \{0,1,\ldots\}$.
So in this classification setting, previous observations may be used to classify the present observation.
If $(X_0,Y_0)$ is independent of the past $(X_{-n},Y_{-n})_{n \geq 1 }$ then clearly previous observations carry no information on $Y_0$
and our optimal classification
would be given by $\argmax_{y} P(Y_0=y|X_0=x_0)$.

But in a variety of  classification problems we encounter dependence. Looking e.g. at the on-line
classification of handwritten characters  the dependence structure in natural language
could be taken into account. In this situation, $X_0$ would be the current handwritten character to be classified, $Y_0$ the unknown true character, the foregoing handwritten character would be $X_{-1}$ and the unknown true character $Y_{-1}$, and in general $X_{-n}$ would be the $n$-th one preceding $X_0$ with unknown $Y_n$. So, there is a well-known dependence between the $Y_n$'s, described by linguists using Markov models (see, e.g., \cite{DefenseAnalyses.CommunicationsResearchDivision1980} for early discussions), 
 and this dependence is of course inherited by the $X_n$'s. A popular model for this situation is given by a hidden Markov model, which we shall also use in this paper. 

Coming back to the general model, we prescribe to use the present and in addition the last $l$
preceding observables.
Then a classification rule with memory $l$  takes  the form
$$ (x_0,x_{-1},\ldots,x_{-l}) \mapsto \delta(x_0,x_{-1},\ldots,x_{-l})$$
for some measurable $\delta:{\cal X}^{l+1} \to M$. The optimal rule is given by
$$ (x_0,x_{-1},\ldots,x_{-l}) \mapsto \argmax_{y}P(Y_0=y| X_0=x_0,X_{-1}=x_{-1},\ldots,X_{-l}=x_{-l})$$
with Bayes risk
$$R_l =  \int \min_y P(Y_0 \neq y| X_0,X_{-1},\ldots,X_{-l}) dP.$$
Obviously
$$ R_0 \geq R_1 \geq \cdots  \geq R_l \geq \cdots $$
Assume that we have such a process $(X_{-n},Y_{-n})_{n \in \N_0}$  with full
past for our mathematical model.
By martingale convergence it follows that for $l \to \infty$
$$ R_l \to   R^*=\int \min_y P(Y_0 \neq y| (X_{-n})_{ n \in \N_0}) dP.$$

Here it is important to point out that this paper centers around the behaviour of the optimal classification procedure in dependence on $l$, the number of past observations used. This differs markedly from one of the main lines of research in the i.i.d. model of supervised learning where the focus is on the behaviour of classification procedures in dependence on $n$, the size of the training sequence. 

To investigate the behaviour of procedures which
 incorporate preceding  information into
classification rules  we will use the setting of hidden Markov models.
  This class of models has been of considerable  interest in the
theory and applications of pattern recognition, see
the monographs by \cite{MR1692202} and by
\cite{huang1990hidden} from a more practical viewpoint. It provides a class which allows
good modelling for various problems with dependence and still may be handled well from
the analytical, the algorithmic, and the statistical point of view, see the
monograph by \cite{cappe2005inference}. The applications range from
biology to speech recognition to finance; the above monographs contain a wealth of
such examples.

A  theoretical contribution to pattern recognition for such models was given by \cite{MR1873337}
where the asymptotic risk of misclassification for nearest neighbor rules in dependent models
including hidden
Markov models was derived. Similar models were treated in \cite{MR2274382} to obtain consistency
for certain classes of procedures, i.e. convergence of the risk of misclassification to the
Bayes risk. As consistency for classification follows from consistency in the corresponding
regression problem, see e.g. \cite[6.7]{gyorfi1996probabilistic}, any result on regression consistency yields a result on classification consistency, and a wealth of such results is available, e.g. under mixing conditions.
All these results invoke the convergence of the size $n$ of a training sequence to infinity and do
not cover the topic of this paper. Closer to our paper is the problem of predicting $Y_0$ from
$(X_0,Y_{-1},\ldots,Y_{-l})$ for stationary and ergodic time series, see e.g.
\cite[Chapter 27]{MR1920390}. Our treatment differs as we do not have
knowledge (just guesses of) $(Y_{-1},\ldots,Y_{-l})$ in on-line pattern recognition, only that of
$(X_0,X_{-1},\ldots,X_{-l})$.

The hidden Markov model as it will be used in this paper takes the following form.
We assume that for each $m$ we have, written in their temporal order,
 observables  $X_{-m},X_{-m+1},\ldots, X_{-1}, X_0$ and
unobservables $Y_{-m},Y_{-m+1},\ldots, Y_{-1}, Y_0$. The unobservables form a Markov chain.
The observables are conditionally independent given the unobservables in the form of
$$ P(X_0 \in B_0,\ldots, X_{-m} \in B_{-m}|Y_0=y_0,\ldots, Y_{-m}=y_{-m})
=Q(B_0,y_0)\cdots  Q(B_{-m},y_{-m})$$
for some stochastic kernel $Q$ and are not Markovian in general.
 This stochastic kernel and the transition matrix of the chain
are assumed to be the same for each $m$. But we allow for the flexibility that, for each $m$,
a different initial distribution, i.e. distribution of $Y_{-m}$ may occur. Note that $m$
stands for the time point in the past  where our model would be started and the distribution
of $Y_{-m}$ would not be known.

For being completely precise  we would have to use the notation
$Y^{(m)}_{-m},\ldots, Y^{(m)}_0$ since, due to our flexibility in initial distribution,
the distribution of $Y^{(m)}_{-m},\ldots, Y^{(m)}_0$ and $Y^{(m+1)}_{-m},\ldots, Y^{(m+1)}_0$
need not coincide. Hence also $R_l \geq R_{l+1}$ does not hold in general where $R_l$ is computed
in a  model started at some time $-m,m\geq l$, and $R_{l+1}$ in a  model with a possibly different
initial distribution.
But all our bounds will only involve the transition matrix and the stochastic kernel
which do not depend on the index $m$. So we
shall omit this upper index in order not to overburden our notations.

We assume that the transition matrix of the chain is such that there exists a unique stationary probability distribution $\pi$, characterized by the property that if $Y_{-m}$ has the distribution $\pi$ then all later $Y_{-m+k},~k\geq 0,$ have the same distribution $\pi$. For our chain with full past 
$(X_{-n},Y_{-n})_{n \in \N_0}$ 
we consider the stationary setting where each $Y_{-n}$ has the same distribution $\pi$.
Then
of course $R_l^* \geq R_{l+1}^*$ and $\lim_{l}R_l^* = R^*$ denoting the risk in the stationary case with
an additional $^*$.

Without loss of generality we let the probability measures $Q(\cdot,y)$ be given by densities $f_y$
with respect to some $\sigma$-finite measure $\mu$ on ${\cal X}$. So we have for all $n$
$$P(X_{-n} \in B|Y_{-n}=y) = \int_B f_y(x)\mu(dx).$$
This provides a unified treatment for  the case of discrete ${\cal X}$ where $\mu$ might  be the counting measure,
and for the case of Lebesgue densities where ${\cal X} = \R^d$ and $\mu$ might be
 $d$-dimensional Lebesgue measure.

In Section \ref{sec:exp} we shall show under a suitable assumption that
$\lim_l R_l$ exists and is independent of  the particular sequence of initial distributions, hence
$\lim_l R_l = R^*$. Furthermore this
convergence is exponentially fast
and we provide a bound for  $\sup_k |R_l - R_{l+k}|$ in this respect. Let us remark that, as we are looking backwards in time,
the usual geometric ergodicity forward in time does not seem to yield an immediate proof. In  Section \ref{sec:kernel} we introduce kernel classification
rules with memory and discuss their theoretical and practical performance.
Our findings  indicate  that
 it might be useful to include a small number $l$ of  preceding
observations, starting with $l=1$,  to increase the
performance of classification rules with  an acceptable increase in computational complexity.
Various technical proofs are given in Section \ref{sec:proof}.

\section{Exponential Convergence}\label{sec:exp}

We consider a hidden Markov model as described in the Introduction.
For this model we make the following assumption:
\\

$(A)$ All entries $p_{ij}, i,j \in M$, in the transition matrix are $ >0$. All densities are $ >0$ on ${\cal X}$.
\\

$(A)$ will be assumed to hold throughout Sections \ref{sec:exp} and \ref{sec:proof}. It implies the finiteness of the following quantities which will be used in our bounds.

\begin{remark}
Set
\[
\alpha =\max_{\iota ,\kappa ,i,j\in M}\frac{p_{i\iota }p_{\iota j}}
{p_{i\kappa }p_{\kappa j}}, \;\;
\alpha (x)=\max_{ \iota ,\kappa ,i,j\in M}\frac{p_{i\iota }p_{\iota j}\; f_{\iota }(x)}
{p_{i\kappa }p_{\kappa j}\; f_{\kappa }(x)}\mbox{ for } x \in {\cal X}.
\]
Then $\; 1 \leq \alpha, \alpha (x) < \infty$.

Furthermore, with $|M|$ denoting the number of classes,  let
\[
\eta =(1+(|M|-1)\alpha )^{-1},
\;\;  \eta (x)=(1+(|M|-1)\alpha (x))^{-1} \mbox{ for } x \in {\cal X}.\]
Then $\; 0 < \eta, \eta (x) \leq  1/2$.
\end{remark}

The following result provides the main technical tool. Its proof will be given in Section \ref{sec:proof}.
We use the notation  $ x_{-n}^{0}$ for \( (x_{0},\ldots ,x_{-n}) \)
and in the same manner
we use $ X_{-n}^{0}$.

\begin{theorem}\label{MayMin}
Let $l,n \in \N$. Consider a hidden Markov model which starts in some time point $-m < \min \{-l, -n\}$.
Let  \( A\subseteq M \) and fix \( x_{0},\ldots ,x_{-n}\in {\cal X}  \).
Set
\[
m_{l}^{+}=\max _{i\in M}P(Y_{0}\in A|X_{-n}^{0}=x_{-n}^{0},Y_{-l}=i),\, l\in \N \]
and
\[
m_{l}^{-}=\min _{i\in M}P(Y_{0}\in A|X_{-n}^{0}=x_{-n}^{0},Y_{-l}=i),\,  l\in \N \, .\]
Then for all \( l\in \N  \)
\[
m_{l}^{+}-m_{l}^{-}\leq \prod _{j=-l+1}^{0}(1-2\hat{\eta }_{j})\, \, \, ,\]
where \( \hat{\eta }_{j} = \eta (x_{j}) \) for \(  j\in \{-n,\ldots ,0\}\) and
\( \hat{\eta }_{j} = \eta  \) for \( j\notin \{-n,\ldots ,0\} \).

\end{theorem}

In the following corollary the probabilities $P(Y_{0}\in A|X_{-l}^{0}=x_{-l}^{0})$ and
$P(Y_{0}\in A|X_{-l-k}^{0}=x_{-l-k}^{0})$ are treated. The first will pertain to a hidden Markov model
which starts
in some time point $< -l$, the second to one which starts in some time point $< -l-k$. Note that terms
 of the form $P(Y_{0}\in A|X_{-l}^{0}=x_{-l}^{0},Y_{-l-1}=i)$ are identical in both models due to the
identical transition matrix
and the identical kernel $Q$.

\begin{corollary}\label{cor:23}
Let \( l,k\in \N  \), \( A\subseteq M \) and \( x_{0},x_{-1},\ldots , x_{-l-k} \in {\cal X}  \).
Then

\[
|P(Y_{0}\in A|X_{-l}^{0}=x_{-l}^{0})-P(Y_{0}\in A|X_{-l-k}^{0}=x_{-l-k}^{0})|
\leq \prod _{j=-l}^{0}(1-2\eta (x_{j}))\, .\]

\end{corollary}
\begin{proof}
We obtain
\begin{eqnarray*}
&& |P(Y_{0}\in A|X_{-l}^{0}=x_{-l}^{0})-P(Y_{0}\in A|X_{-l-k}^{0}=x_{-l-k}^{0})|\\
 & = & |\sum _{\iota \in M}P(Y_{0}\in A|X_{-l}^{0}=x_{-l}^{0},Y_{-l-1}=\iota )P(Y_{-l-1}=\iota |X_{-l}^{0}=x_{-l}^{0})\\
 &  & \, \, \, \, \, -\sum _{\kappa \in M}P(Y_{0}\in A|X_{-l-k}^{0}=x_{-l-k}^{0},Y_{-l-1}=\kappa )P(Y_{-l-1}=\kappa |X_{-l-k}^{0}=x_{-l-k}^{0})| \\
 &  = & |\sum _{\iota \in M}P(Y_{0}\in A|X_{-l}^{0}=x_{-l}^{0},Y_{-l-1}=\iota )
P(Y_{-l-1}=\iota |X_{-l}^{0}=x_{-l}^{0})\\
 &  & \, \, \, \, \, -\sum _{\kappa \in M}P(Y_{0}\in A|X_{-l}^{0}=x_{-l}^{0},Y_{-l-1}=\kappa )
P(Y_{-l-1}=\kappa |X_{-l-k}^{0}=x_{-l-k}^{0})|\\
   &  \leq  & \max_{i\in M}P(Y_{0}\in A|X_{-l}^{0}=x_{-l}^{0},Y_{-l-1}=i)\\
 &  & \, \, \, \, \, -\min _{i\in M}P(Y_{0}\in A|X_{-l}^{0}=x_{-l}^{0},Y_{-l-1}=i)\\
 &  = & m_{l+1}^{+}-m_{l+1}^{-}
  \leq \prod _{j=-l}^{0}(1-2\eta(x_{j}))\,.
\end{eqnarray*}
using Theorem \ref{MayMin}.

\end{proof}
We now introduce the constants used  for the  exponential bound.

\begin{remark}\label{rem:const}
\begin{enumerate}[(i)]
\item
Set
$$ \beta =\min_{k \in M} \int \frac{1}{ 1 + (|M|-1)\alpha(x)} f_{k}(x)\mu(dx) \;\mbox{ and }\;
\gamma =1-2\beta \,. $$
Then
$$ 0 < \beta \leq \frac{1}{2} \; \mbox{ and }\; 0 \leq \gamma < 1.$$
For all $k \in M$
$$ \int (1 - 2\eta(x)) f_{k}(x)\mu(dx) \leq \gamma.$$
\item For the following result we need additional constants which arise from basic Markov process theory. A transition matrix $Q$ is called uniformly ergodic if there exists a unique stationary probability distribution $\pi$ and there exist constants $a>0,~0<b<1,$ such that for any Markov chain $(Y_n)_{n\geq t}$ with transition matrix $Q$ and any initial distribution at time $t$
\[\|P^{Y_{t+k}}-\pi\|\leq a\cdot b^k,~~k\in\N,\]
in total variation norm $\|\cdot\|$. With the same meaning, also the process $(Y_n)_{n\geq t}$ is called uniformly ergodic. Assumption $(A)$ above implies uniform ergodicity, so that we have for each Markov chain constants $a,b$ as above, see \cite[Chapter 16]{Meyn2012} for a general treatment.
\end{enumerate}
\end{remark}

%
%

\begin{theorem}\label{thm:25}
There exist constants $a>0,~0<b,\gamma<1$ such that for all \( l,k\in \N  \)
\[
  | R_l -R_{l+k}| \leq \gamma ^{l+1}\, ,\]
if $R_l$ and $R_{l+k}$ come from the same model started at some time point $< -l-k$,
$$
| R_l -R_{l+k}| \leq 2 ( \gamma^{l/2} + ab^{l/2}) \, ,$$
in the general case that  $R_l$ and $R_{l+k}$ come from possibly different models, the first started in some time point $< -l$, the second in some time point $< -l-k$.
\end{theorem}
\begin{proof}
The constants $a,b,\gamma$ will be those introduced in Remark \ref{rem:const}.\\
Let us firstly consider the case that  $R_l$ and $R_{l+k}$ stem
 from the same model.
Using  the generic symbol $f$ to denote densities in this  model we  write the joint density as
$f(x_{-l-k}^{0})$ and
the joint conditional density as $f_{y_{-l-k}^{0}}(x_{-l-k}^{0})$. With this notation
we have
\[
f(x_{-l-k}^{0})=\sum _{y_{-l-k}^{0}}P(Y_{-l-k}^{0}=y_{-l-k}^{0})\; f_{y_{-l-k}^{0}}(x_{-l-k}^{0})
\mbox{ and }
f_{y_{-l-k}^{0}}(x_{-l-k}^{0})=\prod _{j=-l-k}^{0}f_{y_{j}}(x_{j})\,,\] furthermore
$$ R_l = \int _{{\cal X}^{l+k+1}}\min _{i\in M}P(Y_{0}\neq i|X_{-l}^{0}=x_{-l}^{0})\;
f(x_{-l-k}^{0})\mu^{l+k+1}(dx_{-l-k}^{0}),$$
$$ R_{l+k}=  -\int _{{\cal X} ^{l+k+1}}\min _{i\in M}P(Y_{0}\neq i|X_{-l-k}^{0}=x_{-l-k}^{0})
\; f(x_{-l-k}^{0})\mu^{l+k+1}(dx_{-l-k}^{0}).$$

We obtain from Corollary \ref{cor:23}  and conditional independence
\begin{eqnarray*}
&&  | R_{l}-R_{l+k} |\\
& \leq  & \int _{{\cal X} ^{l+k+1}}\max_{i\in M}|P(Y_{0} \neq i|X_{-l}^{0}=x_{-l}^{0})-
 P(Y_{0} \neq i|X_{-l-k}^{0}=x_{-l-k}^{0})|\; f(x_{-l-k}^{0})\mu^{l+k+1}(dx_{-l-k}^{0})\\
& \leq  & \int _{{\cal X} ^{l+1}}\prod _{j=-l}^{0}(1-2\eta (x_{j}))\; f(x_{-l}^{0})\mu^{l+1}(dx_{-l}^{0})\\
 & = & \sum _{y_{-l}^{0}}[ P(Y_{-l}^{0}=y_{-l}^{0})\int _{{\cal X} ^{l+1}}
\prod _{j=-l}^{0}(1-2\eta (x_{j}))
 \; f_{y_{-l}^{0}}(x_{-l}^{0})\mu^{l+1}(dx_{-l}^{0}) ] \\
 &  = & \sum _{y_{-l}^{0}}[ P(Y_{-l}^{0}=y_{-l}^{0})\prod _{j=-l}^{0}\int _{{\cal X} }(1-2\eta (x_{j}))
f_{y_{j}}(x_{j})\mu(dx_{j})] \\
  &   \leq & \sum _{y_{-l}^{0}}P(Y_{-l}^{0}=y_{-l}^{0}) \gamma ^{l+1} = \gamma ^{l+1}\,.\\
\end{eqnarray*}
Let us now look at the general case with models 1 and 2, $R_l=R_l^1$ stemming from model 1,
$R_{l+k}= R_{l+k}^2$ from model 2 respectively. Then
$$|R_l^1 - R_{l+k}^2| \leq |R_l^1 - R^1_{\lfloor l/2 \rfloor| }|
                   + |  R^1_{\lfloor l/2 \rfloor} -  R^2_{\lfloor l/2 \rfloor}|
                  +   |  R^2_{\lfloor l/2 \rfloor} -  R^2_{ l + k}|.$$
From the first part of the assertion
$$  |R_l^1 - R^1_{\lfloor l/2 \rfloor| }|
                +   |  R^2_{\lfloor l/2 \rfloor} -  R^2_{ l + k}|
                \leq 2 \gamma^{ l/2}.$$
To treat  $R^1_{\lfloor l/2 \rfloor}$ and   $  R^2_{\lfloor l/2 \rfloor}$ we note that the
conditional Bayes risks for time lag $\lfloor l/2 \rfloor$
given $Y_{-\lfloor l/2 \rfloor}$
are the same in both models hence the
unconditional risks differ by at most the total variation distance
between the two distributions of  $Y_{\lfloor l/2 \rfloor}$ in the two
models. This quantity is  $\leq 2ab^{l/2}$ since both models have been running
for at least $l - \lfloor l/2 \rfloor$ time points, hence
$$  |  R^1_{\lfloor l/2 \rfloor} -  R^2_{\lfloor l/2 \rfloor}| \leq  2ab^{l/2}.$$
\end{proof}
From this we easily obtain our main result.

\begin{theorem}
There exist constants $a>0,~0<b,\gamma<1$ such that for all $l \in \N$
$$ |  R^{*} - R_l | \leq 2 ( \gamma^{l/2} + ab^{l/2}) \, , $$
in particular for $l \to \infty$
$$ R_l \to R^{*}\;.$$
\end{theorem}
\begin{proof}
As in Theorem \ref{thm:25}, the constants $a,b,\gamma$ are those of Remark \ref{rem:const}.
Recall that $R^{*}_l$ is the Bayes risk in the  stationary case.
As already stated earlier,  $\lim _{l} R^{*}_l = R^{*}$ by martingale convergence.
 Theorem \ref{thm:25} shows
$$ |  R_l - R^{*}_{l+k}| \leq  2( \gamma^{l/2} + ab^{l/2})$$
for all $k$, proving  the assertion.

\end{proof}

\section{Kernel Classification With Memory}\label{sec:kernel}

Optimal classification procedures provide benchmarks for the actual behaviour of data driven classification procedures which do not require knowledge of the underlying distribution.
A general principle from statistical classification involves the availability  of  a training sequence
$(x'_1,y'_1,\ldots,x'_n,y'_n)$ where the $x'_i$ have been recorded together with the $y'_i$.
This training sequence is
used for the construction of a regression estimator
$$ {\hat p}(y|x;x'_1,y'_1,\ldots,x'_n,y'_n) \mbox{ for } P(Y=y|X=x)$$
which leads to the classification rule
$$ x \mapsto \argmax_y \,{\hat p}(y|x;x'_1,y'_1,\ldots,x'_n,y'_n).$$
When we choose  a kernel
\[ K: {\cal X} \rightarrow [0,\infty ) \]
and use the common kernel regression estimate we arrive at the kernel classification rule
\[  x \mapsto \argmax_{y} \sum_{i=1}^{n} K\left(\frac{x-x'_{i}}{h}\right)1_{\{y'_{i}=y\}}.\]
The asymptotic behaviour, as the size of the training sequence tends to infinity, has been
thoroughly investigated for such classification rules and in particular for kernel classification rules.
In the i.i.d. case or more generally under suitable mixing conditions, such procedures
are risk consistent in the following sense: Kernel classification
rules asymptotically
achieve  the minimal
risk of misclassification  for  ${\cal X}= \R^d$ if
the size $n$ of the training sequence tends to $\infty$ and
$h=h(n)$ satisfies $h(n) \rightarrow 0$ and $nh(n)^{d} \rightarrow \infty$. As remarked in the Introduction, this type of consistency follows from the
consistency of the corresponding regression estimator, hence any result on regression consistency translates into a result on risk consistency.

 It is the aim of this section to discuss the applicability of kernel classification
with memory  in hidden Markov models.
Assume that the training sequence $X'_{1},Y'_{1},\ldots,X'_{n},Y'_{n}$ is generated according
to a hidden Markov model and that there is a sequence of observations
\[ \ldots, X_{-l},X_{-l-1},\ldots  , X_{0}\mbox{ to be classified}\]
which stems from the stationary  hidden Markov model with the same transition matrix and the same kernel and
is stochastically independent of the training sequence.
For the classification of $X_{0} = x_{0}$ the usual kernel classification rule as described
above would classify $x_{0}$ as belonging to the class
\[  \argmax_{y} \sum_{i=1}^{n}
K\left(\frac{x_{0}-X'_{i}}{h}\right)1_{\{Y'_{i}=y\}}.\]
This ignores the Markovian structure, so we want to use memory
as in the optimal classification of the preceding Section \ref{sec:exp}.
We propose the following procedure.
Fix some memory   $l=1,2,\ldots$ prescribing the number of
preceding  observations used in the classification of the current one. Use a kernel
\[ K : {\cal X}^{l+1} \rightarrow [0,\infty) \]
and, assuming a training sequence of size $n+l$,
 classify observation $x_{0}$  as originating from the class
\[ \argmax_{y} \sum_{i=1}^{n} K\left(\frac 1 h
                              ((x_{-l},\ldots,x_{0})-(X'_{i},\ldots,X'_{i+l}))\right)1_{\{Y'_{i+l}=y\}}.\]
Compared to rules
 without memory  the role of $x_{0}$ is taken by $(x_{-l},\ldots,x_{0})$ whereas
the role of $(X'_{i},Y'_{i})$ is taken by $(X'_{i},Y'_{i},\ldots,X'_{i+l},Y'_{i+l})$.

The approach we propose here leads to a
risk consistent procedure for hidden Markov models, i.e.  the risk converges to the corresponding
Bayes risk when, for fixed  $l$,  the size of the training sample
 $n$ tends to $\infty$.
The proof of this risk consistency adapts the methods of proof for the i.i.d. case to the Markov model we have here. We present the basic facts here and refer to \cite{Irle1997} for a detailed treatment; see also \cite[Chapter 13]{gyorfi1989nonparametric}. 

The kernel $K$ has to satisfy that for any $y$
\[\frac{E1_{\{Y=y\}}K(\frac{1}{k}[(x_{-l},...,x_{0})-(X_{-l},...,X_{0})])}{EK(\frac{1}{k}[(x_{-l},...,x_{0})-(X_{-l},...,X_{0})])}\rightarrow P(Y=y|X_{-l}=x_{-l},...,X_{0}=x_{0})\]
for a.a. $(x_{-l},...,x_{0})$ as $k\rightarrow 0$. 
Any kernel $K$ such that $K\geq0$, $K$ is bounded with bounded support, and there exist $t_0,c>0$ such that $K(z)\geq c$ for $\|z\|\leq t_0$, fulfills the above condition, see, e.g., \cite[10.1]{gyorfi1996probabilistic}. We call such a kernel regular.


Next note that we consider a uniformly ergodic transition matrix for our Markov chain. Looking at the hidden Markov model forward in time, the process $(X_n',Y_n')_{n\in\N}$ forms a Markov chain with state space $\mathcal{X}\times M$ in discrete time. The stationary distribution for this process is given by 
\[\pi'(A\times B)=\sum_{y\in B}Q(A,y)\pi(\{y\}).\]
It follows immediately that this process is again uniformly ergodic such that for $a>0,~0<b<1$, the constants for the $Y$-process, it holds that for all $n$
\[\|P^{(X_n',Y_n')}-\pi'\|\leq a\cdot b^n\]
since 
\begin{align*}
|P^{(X_n',Y_n')}(A\times B)-\pi'(A\times B)|&=|\sum_{y\in B}Q(A,y)(P(Y_n'=y)-\pi(\{y\}))|\\
&\leq |P(Y_n'\in B)-\pi(B)|\leq a\cdot b^n.
\end{align*}

In exactly the same manner, the process $(Z_n)_{n\in\N}=(X_n',Y_n',...,X_{n+l}',Y_{n+l}')$ is a Markov chain with state space $\mathcal{Z}=(\mathcal{X}\times M)^{l+1}$ and stationary probability distribution $\pi^{(l)}$ given by
\begin{align*}
\pi^{(l)}\big(\prod_{i=1}^{l+1} (A_i\times B_i) \big)&=\sum_{y_1\in B_1,...,y_{l+1}\in B_{l+1}}\left(\prod_{i=1}^{l+1}Q(A_i,y_i)\right)\pi(\{y_1\})p_{y_1,y_2}...p_{y_l,y_{l+1}},
\end{align*}
the $p$'s denoting the transition probabilities for the original chain. It is easily seen that this process is again uniformly ergodic where, with the same constants $a,b$, we have for all $n$
\[\|P^{Z_n}-\pi^{(l)}\|\leq a\cdot b^n.\]

Finally we note that any uniformly ergodic process is geometrically mixing in the sense that there exist $\alpha>0,~0<\beta<1$ such that for all $n$
\[|Ef(Z_{i+n})g(Z_n)-Ef(Z_{i+n})Eg(Z_n)|\leq \alpha\beta^i\]
for any measurable $f,g:\mathcal{Z}\rightarrow\R,~|f|,~|g|\leq 1,$ and any initial distribution, see \cite[Theorem 16.1.5]{Meyn2012}.


Using the foregoing notations we can obtain the following result. 

\begin{theorem}
Let $K$ be  regular 
and let $h(n)>0, n=1,2,\ldots$ be such that
$h(n) \rightarrow 0$ and $nh(n)^{d(l+1)} \rightarrow \infty$.
Denote the risk of the kernel classification rule by $L_{n}^{(l)}$. Then as $n\rightarrow\infty$
\[  L_{n}^{(l)} \rightarrow R_l^*.\]
\end{theorem}

\begin{proof}
The proof is based on the observation that \emph{classification is easier than regression function estimation}, see \cite[6.7]{gyorfi1996probabilistic}. To adapt this to our setting fix $y\in M$. Let for $(x_{-l},....,x_0)\in \mathcal X^{l+1}$
\[\hat{p}_n(x_{-l},....,x_0)=\frac{\sum_{i=1}^nK\big(\frac{1}{h(n)}((x_{-l},\ldots,x_{0})-(X'_{i},\ldots,X'_{i+l}))\big)1_{\{Y_{i+l}'=y\}}}{\sum_{i=1}^nK\big(\frac{1}{h(n)}((x_{-l},\ldots,x_{0})-(X'_{i},\ldots,X'_{i+l}))\big)}.\]
This is the kernel regression function estimator of size $n$ for
\[p(x_{-l},....,x_0)=P(Y_0=y|(X_{-l},...,X_0)=(x_{-l},....,x_0)),\]
with corresponding kernel classification rule
\[\argmax_y \sum_{i=1}^nK\big(\frac{1}{h(n)}((x_{-l},\ldots,x_{0})-(X'_{i},\ldots,X'_{i+l}))\big)1_{\{Y_{i+l}'=y\}}.\]
Now to show $L_n^{(l)}\rightarrow R^*_l$ it is enough to show that as $n\rightarrow\infty$ 
\begin{align}\label{eq:conv_prob}
\hat{p}_n(x_{-l},....,x_0)\rightarrow p(x_{-l},....,x_0)
\end{align}
in probability for almost all $(x_{-l},....,x_0)$, see \cite[Theorem 6.5]{gyorfi1996probabilistic}. For this we may apply \cite[Theorem 1]{Irle1997} and the application to kernel regression estimators in ibid, part 3, in particular the representation for $\hat{p}_n$, p.138. We than use uniform ergodicity, geometric mixing, regularity of the kernel, together with $nh(n)^{d(l+1)} \rightarrow \infty$ to infer that the conditions to apply \cite[Theorem 1 (i)]{Irle1997} are fulfilled. This then shows the assertion. 
\end{proof}
\begin{remark}
\begin{enumerate}[(i)]
\item The more complicated result of almost sure convergence $\hat{p}_n(x_{-l},....,x_0)\rightarrow p(x_{-l},....,x_0)$ in \eqref{eq:conv_prob} needs additional conditions; see \cite[Corollary 4]{Irle1997}, \cite[Chapter 13]{gyorfi1989nonparametric}. But here, we only need convergence in probability. 
\item This method of using past information to construct a classification procedure seems generally applicable. We simply have to replace $x_0$ by $(x_{-l},...,x_0)$ and the learning sequence $(X_n',Y_n')_n$ by $(X_n',Y_n',...,X_{n+l}',Y_{n+l}')_n$. E.g. for a nearest neighbor classification we would look for the nearest neighbor among the $(X_n',Y_n',...,X_{n+l}',Y_{n+l}')$ and use the resulting $Y_{n+l}'$ for classification. To show consistency, we can proceed in the same way as for kernel classification using the nearest neighbor regression estimate, compare \cite[Part 4]{Irle1997}.
\end{enumerate}
\end{remark}

So asymptotically as $n$ tends to infinity, the kernel classification rule performs as the optimal rule of Section \ref{sec:exp} and may be used as a typical nonparametric rule to test the usefulness of invoking preceding information.

From a practical point of view we now comment
on the performance in simulations and in recognition problems for isolated handwritten
letters which points to a saving in misclassifications.
\subsection*{Performance studies}

In the following we report on some typical results in our studies of the actual behaviour of the
kernel classification rule as proposed in this paper. As a general experience we
point out that memory $l>1$ did not lead to significant improvement over $l=1$ so
that we only compare the cases $l=0,l=1$.

(i)  In simulations 1 and 2 we choose $\ldots Y_{-1}, Y_{0},Y_{1},\ldots$ as a Markov chain with 4 states and transition probabilities
\[ \left[ \begin{array}{cccc}
  0 & 0 & 1 & 0 \\
  0 & 0 & 0 & 1 \\
  0.3 & 0.7 &0 & 0 \\
  0.7 & 0.3 & 0 & 0
\end{array} \right] \]
with stationary distribution  $(0.25,0.25,0.25,0.25)$.

In simulation 3 we choose $Y_{1},Y_{2},\ldots$ i.i.d. following the stationary distribution.
The $X_{i}$'s have a three-dimensional normal distribution with identical covariance matrix and mean vectors
\[ \begin{array}{ll}
(0,0,0)  &  \mbox{ in simulations 1,2,3 for class 1 },\\
(4,0,0)  &  \mbox{ in simulations 1,2,3 for class 2 },\\
(3.9,3.9,0) & \mbox{ in simulation 1, $(4,4,0)$ in simulations 2,3 for class 3,} \\
(0,3.9,0) & \mbox{ in simulation 1, $(3.8,3.8,0)$ in simulations 2,3 for class 4.}\end{array} \]
So there is good distinction between all classes in simulation 1 with easy classification,
there is poor distinction between classes 3 and 4 in simulations 2 and 3.
The following table gives the error rate for classification with size $n$ of the
training sequence in the first row. A normal kernel is used.
\begin{center}
 \begin{tabular}{c|c|c|c|c}
 sim & $l$ & 100 & 300 & 500 \\ \hline
1    & 0 & 0.03 & 0.03 & 0.03 \\
     & 1 & 0.01 & 0.03 & 0.02 \\ \hline
2    & 0 & 0.21 & 0.19 & 0.21 \\
     & 1 & 0.05 & 0.05 & 0.04 \\ \hline
3    & 0 & 0.30 & 0.28 & 0.31 \\
     & 1 & 0.34 & 0.33 & 0.30
\end{tabular}
\end{center}

This shows that use of the Markov structure in simulation 2 through $l=1$ leads to the
possibility of distinguishing between classes 3 and 4. In the i.i.d. case of simulation
3 an appeal to memory of course does not help.

\vspace{5 mm}

(ii) The classification of handwritten isolated capital letters was performed using
kernel methods. Features were obtained by transforming handwritten letters into
$16 \times 16$ grey-value matrices.
The learning sequence was obtained by merging samples from  seven different persons.

The following typical error rates
resulted from the classification of the word SAITE (german for 'string')
where error rates are writer dependent. A normal kernel was used with $h=1.0$ and $h=0.25$.

\begin{center}
 \begin{tabular}{c|c|c|c}

writer & $l$ & 1.0 & 0.25 \\ \hline

   1   & 0 & 0.261 & 0.083 \\
       & 1 & 0.012 & 0.007 \\ \hline
   2   & 0 & 0.334 & 0.115 \\
       & 1 & 0.025 & 0.022  \\ \hline
   3   & 0 & 0.166 & 0.075  \\
       & 1 & 0.030 & 0.024
\end{tabular}
\end{center}

Use of the Markov structure through $l=1$ seems to lead to
improved performance. Of course, the incorporation of memory can be applied
to any procedure of pattern recognition.
In particular we have also looked into nearest neighbor rules with memory $l$.
Our findings have been similar to those for the kernel rule as discussed above and also
advocate the use of memory $l=1$.

\section{Proofs for Section \ref{sec:exp}}\label{sec:proof}

\begin{lemma}\label{lem:41}
Let $l \in \N_0$. Consider a hidden Markov model which starts in some time point $-m < -l$
and let $T \subseteq \{-m,-m+1,\ldots,0\}$.

For any  \( x_{t}\in {\cal X} ,\: t\in T, \)
and  \( i,\iota ,\kappa \in M \) we have

(i) for \( -l\in T \)
\[
\frac{P(Y_{-l}=\iota |Y_{-l-1}=i, X_{t}=x_{t},t\in T)}
{P(Y_{-l}=\kappa |Y_{-l-1}=i, X_{t}=x_{t},t\in T)}\leq \alpha (x_{-l})\, ,\]

(ii) for  \( -l\notin T \)
\[
\frac{P(Y_{-l}=\iota |Y_{-l-1}=i,X_{t}=x_{t}:t\in T)}
{P(Y_{-l}=\kappa |Y_{-l-1}=i, X_{t}=x_{t},t\in T)}\leq \alpha \,  .\]

\end{lemma}
\begin{proof}
We shall use the symbol $f$ in a generic way to denote joint and
conditional joint densities; also we use  \( T_{\geq s}= \{t\in T:t\geq s\} \).
From the properties of a hidden Markov model
we obtain

\begin{eqnarray*}
&&\frac{P(Y_{-l}=\iota |Y_{-l-1}=i,X_{t}=x_{t},t\in T)}{P(Y_{-l}=\kappa |Y_{-l-1}=i,X_{t}=x_{t},t\in T)}\\
  & = & \frac{P(Y_{-l}=\iota |Y_{-l-1}=i,X_{t}=x_{t},t\in T_{\geq -l})}
{P(Y_{-l}=\kappa |Y_{-l-1}=i,X_{t}=x_{t},t\in T_{\geq -l})} \\
  & = & \frac{P(Y_{-l-1}=i,Y_{-l}=\iota |X_{t}=x_{t},t\in T_{\geq -l})}
{P(Y_{-l-1}=i,Y_{-l}=\kappa |X_{t}=x_{t},t\in T_{\geq -l})}\\
  & = & \frac{\sum _{j}P(Y_{-l-1}=i,Y_{-l}=\iota ,Y_{-l+1}=j|X_{t}=x_{t},t\in T_{\geq -l})}
{\sum _{j}P(Y_{-l-1}=i,Y_{-l}=\kappa ,Y_{-l+1}=j|X_{t}=x_{t},t\in T_{\geq -l})} \\
 & = & \frac{\sum _{j}P(Y_{-l-1}=i,Y_{-l}=\iota ,Y_{-l+1}=j)\;
f(x_{t},t\in T_{\geq -l}|Y_{-l-1}=i,Y_{-l}=\iota ,Y_{-l+1}=j)}
{\sum _{j}P(Y_{-l-1}=i,Y_{-l}=\kappa ,Y_{-l+1}=j)\;
f(x_{t},t\in T_{\geq -l}| Y_{-l-1}=i,Y_{-l}=\kappa ,Y_{-l+1}=j)}\\
 & = & \frac{\sum _{j}P(Y_{-l-1}=i)\; p_{i\iota }\; p_{\iota j}\;
f(x_{t},t\in T_{\geq -l}|Y_{-l-1}=i,Y_{-l}=\iota ,Y_{-l+1}=j)}
{\sum _{j}P(Y_{-l-1}=i)\; p_{i\kappa }\; p_{\kappa j}\;
f(x_{t},t\in T_{\geq -l}| Y_{-l-1}=i,Y_{-l}=\kappa ,Y_{-l+1}=j)}\\
 & = & \frac{p_{i\iota }\; \sum _{j}p_{\iota j}\;
f(x_{t},t\in T_{\geq -l}|Y_{-l-1}=i,Y_{-l}=\iota ,Y_{-l+1}=j)}
{p_{i\kappa }\; \sum _{j}p_{\kappa j}\;
f(x_{t},t\in T_{\geq -l}| Y_{-l-1}=i,Y_{-l}=\kappa ,Y_{-l+1}=j)}\\
  & = & \frac{p_{i\iota }\; \sum _{j}p_{\iota j}\; f_{\iota }(x_{-l})\; f_{j}(x_{-l+1})\;
f(x_{t},t\in T_{\geq -l+2}| Y_{-l+1}=j)}
{p_{i\kappa }\; \sum _{j}p_{\kappa j}\; f_{\kappa }(x_{-l})\; f_{j}(x_{-l+1})\;
f(x_{t},t\in T_{\geq -l+2}| Y_{-l+1}=j)}\, \, \, \, \, (*)\\
& \leq  & \frac{p_{i\iota }\; f_{\iota }(x_{-l})}{p_{i\kappa }\; f_{\kappa }(x_{-l})}\;
\max _{j}\frac{p_{\iota j}\; f_{j}(x_{-l+1})\;
f(x_{t},t\in T_{\geq -l+2}| Y_{-l+1}=j)}
{p_{\kappa j}\; f_{j}(x_{-l+1})\;
f(x_{t},t\in T_{\geq -l+2}| Y_{-l+1}=j)}\\
 & = & \frac{p_{i\iota }\; f_{\iota }(x_{-l})}{p_{i\kappa }\; f_{\kappa }(x_{-l})}\;
\max _{j}\frac{p_{\iota j}}{p_{\kappa j}}\\
 & \leq  & \max _{\iota ,\kappa ,i,j}\frac{p_{i\iota }\; p_{\iota j}\; f_{\iota }(x_{-l})}{p_{i\kappa }\; p_{\kappa j}\; f_{\kappa }(x_{-l})}\, .\\
 &
\end{eqnarray*}

Note that the term  \( f_{j}(x_{-l+1}) \) in
line \( (*) \) only appears for  \( -l+1\in T \).
In the second part of the assertion we have \( -l\notin T \) hence the terms  \( f_{\iota }(x_{-l}) \)
and \( f_{\kappa }(x_{-l}) \) do not appear in line \( (*) \). Thus the dependence on
 \( x_{-l} \) disappears.

\end{proof}

As a consequence  we obtain:

\begin{corollary}\label{Korollar}
Consider the situation of Lemma \ref{lem:41}. Then

(i) for  \( -l\in T \)
\[
P(Y_{-l}=j|Y_{-l-1}=i,X_{t}=x_{t},t\in T )\geq \eta (x_{-l})>0\, ,\]

(ii) for  \( -l\notin T \)
\[
P(Y_{-l}=j|Y_{-l-1}=i,X_{t}=x_{t},t\in T )\geq \eta >0\,  .\]

\end{corollary}
\begin{proof}
Set  \( \hat{\alpha }= \alpha (x_{-l})\) for $ -l\in T $ and \( \hat{\alpha }= \alpha\)  for $ -l\notin T$.
Lemma \ref{lem:41}  implies

\begin{eqnarray*}
 && \frac{1}{P(Y_{-l}=j|Y_{-l-1}=i,X_{t}=x_{t},t\in T )}\\
 & = & \sum _{k}\frac{P(Y_{-l}=k|Y_{-l-1}=i,X_{t}=x_{t},t\in T)}{P(Y_{-l}=j|Y_{-l-1}=i,X_{t}=x_{t}:t\in T)}\\
 & \leq  & 1+(m-1)\hat{\alpha }\, ,
\end{eqnarray*}
hence
\[
P(Y_{-l}=j|Y_{-l-1}=i,X_{t}=x_{t},t\in T)\geq (1+(m-1)\hat{\alpha })^{-1}\,.\]

\end{proof}

We now give the proof of Theorem \ref{MayMin}.

\begin{proof}
For  \( A=M \)  we have \( m_{l}^{+}-m_{l}^{-}=1-1=0 \). So assume \( A\neq M \).
We use induction in $l$ and shall apply Corollary \ref{Korollar}.

Let \( l=1 \):
Chose \( j\in A^{c} \).
\begin{eqnarray*}
m_{1}^{+}-m_{1}^{-} & = & \max _{i\in M}P(Y_{0}\in A|X_{-n}^{0}=x_{-n}^{0},Y_{-1}=i)-m_{1}^{-}\\
 & = & 1-\min _{i\in M}P(Y_{0}\notin A|X_{-n}^{0}=x_{-n}^{0},Y_{-1}=i)-m_{1}^{-}\\
 & \leq  & 1-\min _{i\in M}P(Y_{0}=j|X_{-n}^{0}=x_{-n}^{0},Y_{-1}=i)-m_{1}^{-}\\
 & \leq  & 1-\hat{\eta }_{0}-\hat{\eta }_{0} =
(1-2\hat{\eta }_{0})\, .
\end{eqnarray*}
The inductive step:
\\
Assume that the assertion is true for \( l\in \N  \). Let $j^+$ and $j^-$ be such that the maximum and minimum
respectively, are attained in these values, i.e. $j^+ = \arg m_{l}^{+},\, j^- = \arg m_{l}^{-}$.
Then
\begin{eqnarray*}
m_{l+1}^{+} & = & \max_{i\in M}P(Y_{0}\in A|X_{-n}^{0}=x_{-n}^{0},Y_{-l-1}=i)\\
 & = & \max_{i\in M}\{\sum _{j\in M}P(Y_{0}\in A|X_{-n}^{0}=x_{-n}^{0},Y_{-l}=j,Y_{-l-1}=i)\\
 &  & \, \, \, \, \, \; \times P(Y_{-l}=j|X_{-n}^{0}=x_{-n}^{0},Y_{-l-1}=i)\} \\
 & = & \max_{i\in M}\{\sum _{j\in M}P(Y_{0}\in A|X_{-n}^{0}=x_{-n}^{0},Y_{-l}=j)\\
 &  & \, \, \, \, \, \times P(Y_{-l}=j|X_{-n}^{0}=x_{-n}^{0},Y_{-l-1}=i)\}\\
 & = & \max_{i\in M}\{\sum _{j\neq j^{-}}P(Y_{0}\in A|X_{-n}^{0}=x_{-n}^{0},Y_{-l}=j)\\
 &  & \, \, \, \, \, \; \times P(Y_{-l}=j|X_{-n}^{0}=x_{-n}^{0},Y_{-l-1}=i)\\
 &  & \, \, \, \, \, +m_{l}^{-}\; P(Y_{-l}= j^{-} |X_{-n}^{0}=x_{-n}^{0},Y_{-l-1}=i)\} \\
 & \leq  & \max_{i\in M}\{\sum_{j\neq j^{-}} m_{l}^{+}\; P(Y_{-l}=j|X_{-n}^{0}=x_{-n}^{0},Y_{-l-1}=i)\\
 &  & \\
 &  & \, \, \, \, \, +m_{l}^{-}\; P(Y_{-l}= j^{-} |X_{-n}^{0}=x_{-n}^{0},Y_{-l-1}=i)\} \\
 & = & \max_{i\in M}\{m_{l}^{+}\; (1-P(Y_{-l}= j^{-}|X_{-n}^{0}=x_{-n}^{0},Y_{-l-1}=i)) \\
 &  & \, \, \, \, \, +m_{l}^{-}\; \underbrace{P(Y_{-l}= j^{-}|X_{-n}^{0}=x_{-n}^{0},Y_{-l-1}=i)}_{\geq \hat{\eta }_{-l}}\} \\
 & \leq  & \max_{i\in M}\{m_{l}^{+}(1-\hat{\eta }_{-l})+m_{l}^{-}\hat{\eta }_{-l}\}\\
 & = & m_{l}^{+}(1-\hat{\eta }_{-l})+m_{l}^{-}\hat{\eta }_{-l}\, ,
\end{eqnarray*}

Similarly

\begin{eqnarray*}
m_{l+1}^{-} & = & \min _{i\in M}P(Y_{0}\in A|X_{-n}^{0}=x_{-n}^{0},Y_{-l-1}=i)\\
 & = & \min _{i\in M}\{\sum _{j\in M}P(Y_{0}\in A|X_{-n}^{0}=x_{-n}^{0},Y_{-l}=j,Y_{-l-1}=i)\\
 &  & \, \, \, \, \, \; \times P(Y_{-l}=j|X_{-n}^{0}=x_{-n}^{0},Y_{-l-1}=i)\} \\
 & = & \min _{i\in M}\{\sum _{j\in M}P(Y_{0}\in A|X_{-n}^{0}=x_{-n}^{0},Y_{-l}=j)\\
 &  & \, \, \, \, \, \times P(Y_{-l}=j|X_{-n}^{0}=x_{-n}^{0},Y_{-l-1}=i)\} \\
 & = & \min _{i\in M}\{\sum_{j\neq j^{+}}P(Y_{0}\in A|X_{-n}^{0}=x_{-n}^{0},Y_{-l}=j)\\
 &  & \, \, \, \, \, \; \times P(Y_{-l}=j|X_{-n}^{0}=x_{-n}^{0},Y_{-l-1}=i)\\
 &  & \, \, \, \, \, +m_{l}^{+}\; P(Y_{-l}=j^{+}|X_{-n}^{0}=x_{-n}^{0},Y_{-l-1}=i)\}\\
 & \geq  & \min _{i\in M}\{\sum_{j\neq j^{+}}m_{l}^{-}
\; P(Y_{-l}=j|X_{-n}^{0}=x_{-n}^{0},Y_{-l-1}=i)\\
 &  & \, \, \, \, \, +m_{l}^{+}\; P(Y_{-l}= j^{+}|X_{-n}^{0}=x_{-n}^{0},Y_{-l-1}=i)\}\\
 & = & \min _{i\in M}\{m_{l}^{-}\; (1-P(Y_{-l}= j^{+}|X_{-n}^{0}=x_{-n}^{0},Y_{-l-1}=i))\\
 &  & \, \, \, \, \, +m_{l}^{+}\; \underbrace{P(Y_{-l}= j^{+}|X_{-n}^{0}=x_{-n}^{0},Y_{-l-1}=i)}_{\geq \hat{\eta }_{-l}}\} \\
 & \geq  & \min _{i\in M}\{m_{l}^{-}(1-\hat{\eta }_{-l})+m_{l}^{+}\hat{\eta }_{-l}\}\\
 & = & m_{l}^{-}(1-\hat{\eta }_{-l})+m_{l}^{+}\hat{\eta }_{-l}\, .
\end{eqnarray*}

This implies

\begin{eqnarray*}
 m_{l+1}^{+}-m_{l+1}^{-} & \leq  & (1-2\hat{\eta }_{-l})m_{l}^{+}-(1-2\hat{\eta }_{-l})m_{l}^{-}\\
 & \leq  & (1-2\hat{\eta }_{-l})\prod _{j=-l+1}^{0}(1-2\hat{\eta }_{j})\\
 & = & \prod _{j=-l}^{0}(1-2\hat{\eta }_{j})\, .\\
\end{eqnarray*}

\end{proof}

\bibliographystyle{plain}
\bibliography{IWHMM_rev1}

\end{document}